\documentclass{article}

\usepackage[preprint]{neurips_2026}

\usepackage[utf8]{inputenc} 
\usepackage[T1]{fontenc}    
\usepackage{hyperref}       
\usepackage{url}            
\usepackage{booktabs}       
\usepackage{amsfonts}       
\usepackage{nicefrac}       
\usepackage{microtype}      
\usepackage{xcolor}         

\newcommand{\srp}{\textsf{SRP}}
\newcommand{\optinll}{\textsf{PathUCB}} 
\newcommand{\bayesnll}{\textsf{PathTS}} 
\usepackage{amsmath}  
\usepackage{amsthm}   
\usepackage{amssymb}  
\usepackage{algorithm}
\usepackage[noend]{algorithmic}
\usepackage{multirow}
\usepackage{subcaption}
\usepackage{graphicx}
\usepackage{wrapfig}

\theoremstyle{plain}
\newtheorem{theorem}{Theorem}
\newtheorem{lemma}[theorem]{Lemma} 

\theoremstyle{definition}
\newtheorem{definition}{Definition}
\newtheorem{assumption}{Assumption}

\theoremstyle{remark}
\newtheorem{remark}{Remark}

\title{Stochastic Reset Pathfinding: Path-Level Regret for Cascading Bandits over Graph Paths}

\author{%
  Guni Sharon\\
  Computer Science \& Engineering\\
  Texas A\&M University\\
  College Station, TX 77843 \\
  \texttt{guni@tamu.edu} \\
  \And
  Wei Zhang \\
  Computer Science \& Engineering\\
  Texas A\&M University\\
  College Station, TX 77843 \\
  \texttt{komo@tamu.edu} \\
}

\begin{document}

\maketitle

\begin{abstract}
We introduce \emph{Stochastic Reset Pathfinding} (\srp{}), 
an episodic learning problem on a known directed graph with 
unknown stationary edge success probabilities. In each 
episode, the agent commits to a source-to-goal path, and 
any edge failure during execution resets it to the source. 
\srp{} captures settings such as entanglement distribution 
in quantum repeater networks, payment routing on the 
Lightning Network, and delivery in unreliable mesh networks. 
We show that the global-reset structure makes the optimal 
policy open-loop, placing \srp{} within the combinatorial 
cascading bandit (CCB) framework. We propose a Log-Dijkstra 
meta-algorithm with UCB (\optinll{}) and Thompson Sampling 
(\bayesnll{}) instantiations. Our main technical result is 
a path-level regret bound for \optinll{} that decomposes 
regret over suboptimal paths via a per-path \emph{path 
complexity} $C(\pi)$ combining each edge's prefix and 
suffix reliability. The bound is complementary to the 
edge-level CCB bound and more informative on structured 
graphs with polynomially many source-to-goal paths. 
Experiments on quantum-network, layered-DAG, grid-world, 
and Erd\H{o}s-R\'{e}nyi domains support the theory and 
show that \bayesnll{} typically achieves the best 
empirical performance among algorithms tested. We then 
exhibit an adversarial instance on which \bayesnll{} 
fails to converge---consistent with a known exponential 
obstruction for combinatorial Thompson Sampling on 
multiplicative-reward problems. We recommend \bayesnll{} 
as the practical default while cautioning that 
adversarial instances exist.
\end{abstract}

\section{Introduction}
\label{sec:intro}

Several real-world networked decision problems require an 
agent to traverse a known graph whose links may fail 
unpredictably, where any single failure forces a restart 
from the source. 
Examples include entanglement distribution in quantum 
repeater networks~\cite{Wehner2018, 
chakraborty2020entanglement}, payment routing on the 
Lightning Network~\cite{pickhardt2021reliable}, and 
delivery in unreliable wireless mesh 
networks~\cite{talebi2018ssp}. Link success probabilities 
are unknown and must be learned through repeated 
interaction; the agent's task is to find the most reliable 
source-to-goal path while minimizing failed attempts.

We formalize this class of problems as \emph{Stochastic 
Reset Pathfinding} (\srp{}): in each episode, the agent 
selects a path $\pi$, attempts to traverse it edge-by-edge, 
and resets to the source on the first failure. The 
objective is to minimize cumulative regret with respect to 
the most reliable path 
$\pi^* = \arg\max_\pi \prod_{e \in \pi} p_e$. 
\srp{} superficially resembles a stochastic shortest-path 
(SSP) problem, suggesting goal-oriented 
RL~\cite{Barto1993rtdp, Bonet2003lrtdp, 
jafarnia2023posterior} as a candidate solver. We show this 
framing is inappropriate: the global-reset structure makes 
the optimal policy \emph{open-loop} 
(Lemma~\ref{lemma:open_loop}), collapsing SSP's closed-loop 
machinery into combinatorial path search and placing 
\srp{} within the combinatorial cascading bandit (CCB) 
framework of~\cite{kveton2015combinatorial}. Our 
contributions are:
\begin{itemize}
    \item \textbf{Problem reduction.} The optimal policy 
    for \srp{} is an open-loop simple path 
    (Lemma~\ref{lemma:open_loop}), identifiable by 
    Dijkstra on log-transformed edge weights.

    \item \textbf{Algorithms.} A Log-Dijkstra meta-algorithm 
    (Algorithm~\ref{alg:meta}) with UCB and Thompson 
    Sampling instantiations \optinll{} and \bayesnll{}. 
    \optinll{} is closely related to CombCascade~\cite{kveton2015combinatorial}; \bayesnll{} 
    has no CombCascade analogue and is the recommended 
    practical choice in most regimes we test.
    
    \item \textbf{Path-level regret bound.} Theorem~\ref{thm:path_dependent_regret} bounds the 
    regret of \optinll{} by $\sum_{\pi \notin \Psi^*} 
    C(\pi)^2 \ln T / \Delta(\pi)$, where the per-path 
    \emph{path complexity} $C(\pi)$ encodes each edge's 
    prefix reliability (probability of observation) and 
    suffix reliability (downstream impact of estimation 
    error). This decomposition is complementary to the 
    edge-level CombCascade bound and more informative on 
    structured graphs with many source-to-goal 
    paths.
    
    \item \textbf{Empirical study.} Experiments across four  domains validate the theory and exhibit an adversarial  instance on which \bayesnll{} fails---consistent with the $\Omega(2^{k^*})$ obstruction  of~\cite{wang2018thompson} for combinatorial Thompson  Sampling on multiplicative-reward problems.
\end{itemize}

\section{Problem Formulation}

We consider an episodic learning setting where an agent 
navigates a known directed graph with unknown, stationary 
edge success probabilities.

\begin{definition}[\srp{} Environment]
\label{def:srp}
An \srp{} environment is a tuple $(G, v_s, v_g, 
\mathcal{P})$ where $G = (V, E)$ is a \textbf{known} directed graph, 
$v_s \ne v_g \in V$ are \textbf{known} source and goal nodes, and 
$\mathcal{P} = (p_e)_{e \in E} \in [0,1]^{|E|}$ is a 
vector of stationary edge success probabilities \textbf{unknown} to 
the agent.
\end{definition}

A \emph{valid simple path} is a sequence of edges 
$\pi = (e_1, \dots, e_k)$ with $e_i = (v_{i-1}, v_i) \in E$, 
$v_0 = v_s$, $v_k = v_g$, and no repeated nodes; we write 
$|\pi| = k$ and let $\Psi(v_s, v_g)$ denote the set of such 
paths. We assume $\Psi(v_s, v_g) \ne \emptyset$ (the goal 
is reachable).

\textbf{Transition dynamics and feedback.} At episode $t$, 
the agent selects $\pi_t \in \Psi(v_s, v_g)$ and traverses 
it sequentially. Each edge $e_i$ produces an independent 
outcome $X_{e_i} \sim \text{Bernoulli}(p_{e_i})$: on 
success the agent advances, while on failure the episode 
terminates and a new one begins at $v_s$. The agent 
observes outcomes up to and including the first failure 
(or the full path $(X_{e_1}, \dots, X_{e_k})$ if no failure 
occurs).

\begin{assumption}[Non-degenerate edge reliability]
\label{assum:strict_bounds}
There exists a known constant $p_{\min} > 0$ such that 
$p_e \in [p_{\min}, 1)$ for all $e \in E$.
\end{assumption}

\textbf{Objective.} The \emph{reliability} of $\pi$ is 
$P(\pi) = \prod_{e \in \pi} p_e$, and the goal is to 
identify
\begin{equation}
    \pi^* = \arg\max_{\pi \in \Psi(v_s, v_g)} P(\pi).
\end{equation}
Restricting to simple paths is without loss of generality 
(Lemma~\ref{lemma:open_loop}). Performance is measured by 
the expected cumulative pseudo-regret
\begin{equation}
    R(T) = \mathbb{E}\!\left[\sum_{t=1}^T \bigl(p^* - P(\pi_t)\bigr)\right],
    \qquad p^* = \max_{\pi \in \Psi(v_s, v_g)} P(\pi),
\end{equation}
where the expectation is over the agent's path selections, $\pi_t$. The single-episode regret 
$r(\pi_t) = p^* - P(\pi_t)$ is deterministic in $\pi_t$.

\subsection{Structural Properties of \srp{}}

Two structural properties place \srp{} within the 
combinatorial bandit framework: (1) the optimal policy is 
open-loop, and (2) planning reduces to shortest-path search.

\begin{lemma}[Open-Loop Optimality]
\label{lemma:open_loop}
Under Assumption~\ref{assum:strict_bounds} and assuming $\Psi(v_s, v_g) \neq \emptyset$, the optimal 
policy for an \srp{} environment is equivalent to an open-loop simple 
path $\pi^* \in \Psi(v_s, v_g)$.
\end{lemma}
\begin{proof}
The success probability of any randomized policy is an 
expectation over its deterministic realizations, hence no 
larger than that of the best such realization. It therefore 
suffices to consider deterministic policies $\mu : V \to E$, 
which are sufficient for optimality in finite MDPs by 
\cite{puterman2014markov}.

Fix a deterministic $\mu$. If its all-success trajectory 
from $v_s$ does not reach $v_g$, then $P(\mu) = 0$ and 
$\mu$ is dominated by any path in $\Psi(v_s, v_g)$ (which 
exists by goal-reachability and has positive reliability 
by Assumption~\ref{assum:strict_bounds}). Otherwise, $\mu$ 
induces a finite source-to-goal walk 
$\gamma_\mu = (e_1, \dots, e_m)$. Because any failure 
resets the agent to $v_s$ and $\mu$ is stationary, the 
per-attempt success probability of $\mu$ equals the 
probability that all edges in $\gamma_\mu$ succeed:
\[
    P(\mu) = \prod_{i=1}^{m} p_{e_i}.
\]
If $\gamma_\mu$ contains a cycle, say $v_i = v_j$ for some 
$i < j$, deleting $e_{i+1}, \dots, e_j$ yields a shorter 
source-to-goal walk $\gamma'$ with
\[
    P(\gamma') = P(\gamma_\mu) \Big/ 
    \prod_{\ell=i+1}^{j} p_{e_\ell} > P(\gamma_\mu),
\]
where the strict inequality follows from $p_e < 1$ 
(Assumption~\ref{assum:strict_bounds}). Repeating this 
deletion yields a simple path 
$\pi \in \Psi(v_s, v_g)$ with $P(\pi) \ge P(\mu)$. 
Hence every policy is dominated by some simple 
source-to-goal path. Since $\Psi(v_s, v_g)$ is finite and 
nonempty, a maximum-reliability simple path $\pi^*$ 
exists, and committing to $\pi^*$ at the start of each 
episode is an optimal open-loop policy.
\end{proof}

\begin{lemma}[Reduction to Shortest Path]
\label{lem:structural}
Given known edge probabilities $\mathcal{P}$, any optimal path $\pi^* \in \arg\max_{\pi \in \Psi(v_s,v_g)} P(\pi)$
is a shortest path from $v_s$ to $v_g$ under the transformed edge weights
$w(e) = -\log p_e$, and can be found by Dijkstra's algorithm~\cite{dijkstra}.
\end{lemma}
\begin{proof}
For any path $\pi$, 
$-\log P(\pi)=\sum_{e\in\pi}-\log p_e=\sum_{e\in\pi}w(e)$.
Since $-\log$ is decreasing, maximizing $P(\pi)$ is equivalent to minimizing 
$\sum_{e\in\pi}w(e)$. Assumption~\ref{assum:strict_bounds} gives
$w(e)\in(0,-\log p_{\min}]$, so Dijkstra's algorithm solves the
positive-weight shortest-path problem exactly.
\end{proof}

The log-transform is standard in the most-reliable-path 
literature~\cite{Maheshwari1974}; any exact shortest-path 
algorithm may replace Dijkstra without affecting the theoretical
guarantees below. 

\section{Related Work}
\label{sec:related_work}

\paragraph{Combinatorial cascading bandits.}
Cascading bandits were introduced 
by~\cite{kveton2015cascading} for disjunctive top-K 
recommendation and extended to general feasible sets with 
conjunctive (product) rewards by~\cite{kveton2015combinatorial} 
under the name \emph{combinatorial cascading bandits} (CCB). 
\srp{} is an instance of CCB with feasible set $\Psi(v_s, v_g)$. 
The affiliated CombCascade algorithm applies UCB to per-edge success 
probabilities and uses the same ($-\log$) transformation that 
underpins our analysis, yielding the edge-level regret bound 
of Eq.~\ref{eq:combcascade_bound}, against which we compare 
in Section~\ref{sec:theory}. Subsequent work extends CCB to 
contextual features~\cite{li2016contextual}, adversarial 
corruption~\cite{xie2025cascading}, and state-dependent 
RL~\cite{du2024cascading}; these extensions are complementary in scope but do not 
directly apply to \srp{}. A related 
line studies stochastic shortest-path routing as a 
combinatorial bandit with \emph{additive} edge 
costs~\cite{talebi2018ssp, zhu2018routing} which contrasts with \srp{}'s 
multiplicative $\prod_e p_e$ reward.

\paragraph{Thompson Sampling for cascade feedback.}
Two related TS analyses exist but neither covers \srp{} 
directly. Cheung et al.~\cite{zhong2021thompson} prove a regret bound for Beta-Bernoulli TS on disjunctive cascading 
bandits, but their analysis exploits matroid structure that does not generalize to graph paths. Wang \& 
Chen~\cite{wang2018thompson} analyze combinatorial TS for 
general CMAB with non-linear rewards under semi-bandit (not 
cascade) feedback, and exhibit an $\Omega(2^{k^*})$ lower 
bound on the leading regret constant for multiplicative-reward instances, where $k^*=|\pi^*|$. Remark~\ref{rem:ts_regret} discusses the obstacles to extending these to \srp{}.

\paragraph{Goal-oriented reinforcement learning.}
The Stochastic Shortest Path (SSP) 
problem~\cite{cohen2020near, tarbouriech2021ssp, 
jafarnia2023posterior, johnson2026stochastic} models action 
failures as transitions to neighboring ``slip'' states, 
requiring a closed-loop policy $\mu: V \to E$ and yielding 
regret bounds that reflect this closed-loop 
structure~\cite{cohen2020near}. Under \srp{}'s global-reset 
dynamics the optimal policy is open-loop 
(Lemma~\ref{lemma:open_loop}), placing the problem in the 
bandit rather than the RL regime. Existing SSP bounds and 
algorithms (including the posterior-sampling method 
of~\cite{jafarnia2023posterior}, the closest analog to 
\bayesnll{}) are therefore not directly applicable.

\section{A Graph Search Approach for \srp{}}
\label{sec:algorithms}

We frame the learning process (listed in Algorithm~\ref{alg:meta}) as an iterative loop over: (1) estimate 
each edge's reliability, (2) select a path via Dijkstra on 
log-transformed estimates, (3) attempt to traverse it, and (4)
update estimators on the observed prefix. One subtlety: UCB-style 
estimators can exceed $1$ via their confidence bonus, which 
under $-\log$ would yield negative edge weights and 
invalidate Dijkstra. Both instantiations defined below 
therefore clamp $\hat{p}_e(t) \le 1$, ensuring 
$w_t(e) \ge 0$ and that Dijkstra returns 
$\arg\max_{\pi \in \Psi(v_s, v_g)} \prod_{e \in \pi} \hat{p}_e(t)$.

\begin{algorithm}[ht]
\caption{Log-Dijkstra Meta-Algorithm}
\label{alg:meta}
\begin{algorithmic}[1]
\REQUIRE Graph $G = (V, E)$, source $v_s$, goal $v_g$, horizon $T$
\STATE Initialize estimator state for all $e \in E$ 
\FOR{$t = 1, \dots, T$}
    \STATE $\pi_t \gets \mathrm{Dijkstra}(G, w_t, v_s, v_g)$ with $w_t(e) = -\log \hat{p}_e(t)$ 
    \label{ln:Dijkstra}
    \STATE Traverse $\pi_t$ and let $K$ be the index of the first failure, or $K = |\pi_t|$ if all succeed
    \STATE Update $\hat{p}_{e_i}$ from $X_{e_i}$ for $i = 1, \dots, K$ \label{ln:update_estimators}
\ENDFOR
\end{algorithmic}
\end{algorithm}

The two instantiations below differ only in the estimator 
update at Line~\ref{ln:update_estimators}.

\subsection{\optinll{}: Upper Confidence Bound Exploration}

\optinll{} follows the optimism-in-the-face-of-uncertainty 
principle~\cite{ucb, lai1985asymptotically}. Let $N_e(t)$ 
and $S_e(t)$ denote the number of attempts and successes of 
edge $e$ prior to episode $t$, with empirical mean 
$\bar{p}_e(t) = S_e(t)/N_e(t)$. The estimator is
\begin{equation}
    \hat{p}_e(t) = 
    \begin{cases} 
      1 & N_e(t) = 0, \\
      \mathrm{clip}\!\left(\bar{p}_e(t) + \sqrt{\tfrac{\rho \ln t}{N_e(t)}},\, p_{\min},\, 1\right) & N_e(t) > 0,
   \end{cases}
\end{equation}
where $\rho > 0$ is an exploration parameter ($\rho \ge 2$ 
suffices for Theorem~\ref{thm:path_dependent_regret}). The 
unvisited initialization $\hat p_e(t) = 1$ yields 
$w_t(e) = 0$, prioritizing unexplored edges in Dijkstra. 
The upper clamp at $1$ preserves optimism (since 
$p_e < 1$), and the lower clamp at $p_{\min}$ keeps 
$w_t(e)$ finite; the latter is inactive on the optimism 
event below and acts only as a safety net.

We define the optimism event 
$\mathcal{O}_t = \{\hat{p}_e(t) \ge p_e \text{ for all } e \in E\}$. 
The regret analysis (Section~\ref{sec:path_dependent}) 
partitions on $\mathcal{O}_t$: under optimism, suboptimal selection implies an overestimation gap; the rare complementary event
($\mathcal{O}_t^c$) contributes a bounded regret constant via Lemma~\ref{lem:ucb_optimism}.

\subsection{\bayesnll{}: Thompson Sampling}
\label{sec:bayesnll}

\bayesnll{} samples each edge's reliability from its 
posterior under a Beta-Bernoulli conjugate 
model~\cite{thompson1933likelihood, daniel2018tutorial, 
degroot2005optimal} and selects the path maximizing the 
sampled product. With prior 
$\text{Beta}(\alpha_{e,0}, \beta_{e,0})$ (default 
$\alpha_{e,0} = \beta_{e,0} = 1$) and the standard update 
$\alpha_e \mathrel{+}= X_e$, $\beta_e \mathrel{+}= 1 - X_e$, 
the per-episode estimator is
\begin{equation}
    \hat{p}_e(t) = \max\!\left(\tilde{p}_e(t),\, p_{\min}\right),
    \qquad \tilde{p}_e(t) \sim \text{Beta}(\alpha_e(t), \beta_e(t)).
\end{equation}
Since the Beta is supported on $[0,1]$, the upper clamp at 
$1$ is automatic. The lower clamp at $p_{\min}$ is a safety 
net: as $N_e(t) \to \infty$ the posterior concentrates 
around $p_e \ge p_{\min}$ and the clamp is asymptotically 
inactive; for finite $t$ it introduces a small upward 
bias on data-starved edges without affecting empirical 
behavior (Section~\ref{sec:experiments}).

\begin{remark}[Theoretical Status of \bayesnll{}]
\label{rem:ts_regret}
Despite its strong empirical performance 
(Section~\ref{sec:experiments}), \bayesnll{} is provided 
without a regret bound. Existing TS analyses for cascading 
bandits~\cite{zhong2021thompson} rely on matroid structure 
that fails for general graph paths, and TS analyses for 
combinatorial multiplicative 
rewards~\cite{wang2018thompson} target semi-bandit rather 
than cascade feedback and exhibit an $\Omega(2^{k^*})$ 
obstruction. Motivated by these 
limitations, we design an adversarial instance on which 
\bayesnll{} fails to converge. Appendix~\ref{app:ts_discussion} 
discusses the existing TS theory gaps in detail.
\end{remark}

\section{Theoretical Analysis}
\label{sec:theory}

\srp{} is an instance of the combinatorial cascading bandit 
(CCB) framework~\cite{kveton2015combinatorial} with 
feasible set $\Psi(v_s, v_g)$. 
Applying CombCascade gives the edge-level 
baseline in Eq.~\ref{eq:combcascade_bound}. We provide a complementary path-level regret bound for 
\optinll{}, which uses the same log-transformed path oracle but 
analyzes regret at the level of selected paths rather than 
individual edges. The implementation differences between 
CombCascade and \optinll{} are discussed in 
Appendix~\ref{app:combcascade_diffs}.

\begin{definition}[Per-Edge Gap]
\label{def:edge_gap}
Let $\Psi^* = \arg\max_{\pi \in \Psi(v_s,v_g)} P(\pi)$ and 
$\tilde{E} = \{e \in E : e \notin \pi \text{ for all } 
\pi \in \Psi^*\}$. For $e \in \tilde{E}$, the 
\emph{per-edge suboptimality gap} is
\begin{equation}
    \Delta_e^{\mathrm{CCB}} \;=\; p^* - 
    \max_{\pi \in \Psi(v_s,v_g) :\, e \in \pi,\, 
    P(\pi) < p^*} P(\pi).
\end{equation}
\end{definition}

Letting $L_{\max} = \max_{\pi \in \Psi(v_s,v_g)} |\pi|$, 
Theorem~1 of~\cite{kveton2015combinatorial} translates to:
\begin{equation}
    \mathbb{E}[R(T)] \;\le\; 
    \frac{L_{\max}}{p^*} \sum_{e \in \tilde{E}} 
    \frac{4272}{\Delta_e^{\mathrm{CCB}}} \ln T \;+\; 
    \frac{\pi^2}{3}|E|.
    \label{eq:combcascade_bound}
\end{equation}
This bound is \emph{edge-level}: it sums per-edge 
contributions $\Delta_e^{\mathrm{CCB}}$ without tracking 
where in a path an edge appears. 
Section~\ref{sec:path_dependent} develops a complementary 
\emph{path-level} bound via a per-path complexity 
$C(\pi)$ that captures each edge's prefix reliability 
(probability of being observed) and suffix reliability 
(downstream impact of estimation error). 
Section~\ref{sec:theory_discussion} compares the two 
bounds analytically; Section~\ref{sec:exp1} compares them 
empirically.

\subsection{Path-Dependent Regret Bound}
\label{sec:path_dependent}

\begin{definition}[Per-Edge Overestimation]
\label{def:overestimation}
The \emph{per-edge overestimation} is 
$\delta_e(t) = \hat{p}_e(t) - p_e$, which is non-negative 
on $\mathcal{O}_t$. For unvisited edges, 
$\delta_e(t) = 1 - p_e \le 1 - p_{\min}$.
\end{definition}

\begin{definition}[Suffix Reliability]
\label{def:suffix_reliability}
For a path $\pi = (e_1, \dots, e_L)$ and position 
$i \in \{1, \dots, L\}$, the \emph{suffix reliability} 
after position $i$ is $S_i(\pi) = \prod_{k=i+1}^{L} p_{e_k}$,
with the convention $S_L(\pi) = 1$ (empty product).
\end{definition}

For a suboptimal path $\pi \notin \Psi^*$, we write 
$\Delta(\pi) = p^* - P(\pi) > 0$ for the suboptimality 
gap, and let 
$\hat{P}_t(\pi) = \prod_{e \in \pi} \hat{p}_e(t)$ denote 
the estimated path reliability at episode $t$.

\textbf{Estimation Properties.}
Standard Hoeffding-based concentration arguments give two 
facts (Lemmas~\ref{lem:ucb_optimism} 
and~\ref{lem:ucb_convergence} in 
Appendix~\ref{app:concentration}): 
(i) $\Pr(\mathcal{O}_t) \ge 1 - |E|\,t^{-3}$, contributing 
$O(|E|)$ regret on $\mathcal{O}_t^c$; and 
(ii) for any edge with $N_e(t) > 0$, $\delta_e(t) \le 
2\sqrt{\rho \ln t / N_e(t)}$ with probability 
$\ge 1 - t^{-3}$.

\textbf{Suboptimal Selection and Weighted Overestimation.}
Our next lemma is the structural step where the path-level 
analysis diverges from CombCascade. Lemma~2 
of~\cite{kveton2015combinatorial} bounds the product 
difference $\prod_e \min(\bar{p}_e + u_e, 1) - \prod_e p_e$ 
by the unweighted sum $\sum_e u_e$ of confidence radii; we 
instead weight each edge's overestimation by its suffix 
reliability $S_i(\pi_t)$, exposing position-dependent structure.

\begin{lemma}[Suboptimal Selection Implies Weighted 
Overestimation]
\label{lem:subopt_implies_overest}
Under Assumption~\ref{assum:strict_bounds}, suppose the 
optimism event $\mathcal{O}_t$ holds at episode $t$, and a 
suboptimal path $\pi_t = (e_1, \dots, e_L) \notin \Psi^*$ 
is selected. Then
\begin{equation}
    \sum_{i=1}^{L} S_i(\pi_t)\,\delta_{e_i}(t) \;\ge\; 
    \Delta(\pi_t).
    \label{eq:weighted_overest}
\end{equation}
\end{lemma}
\begin{proof}
On $\mathcal{O}_t$, every optimal path $\pi^* \in \Psi^*$ 
satisfies $\hat{P}_t(\pi^*) \ge p^*$. Since the algorithm 
selects  $\pi_t \in \arg\max_{\pi\in\Psi(v_s,v_g)} \hat P_t(\pi)$ via Dijkstra 
(Lemma~\ref{lem:structural}), 
$\hat{P}_t(\pi_t) \ge \hat{P}_t(\pi^*) \ge p^*$, so
\begin{equation}
    \Delta(\pi_t) = p^* - P(\pi_t) \;\le\; 
    \hat{P}_t(\pi_t) - P(\pi_t).
\end{equation}
The telescoping identity 
$\prod_i a_i - \prod_i b_i = \sum_i (a_i - b_i) 
\prod_{j<i} a_j \prod_{k>i} b_k$ applied to 
$a_i = \hat p_{e_i}(t)$ and $b_i = p_{e_i}$ gives
\begin{equation}
    \hat{P}_t(\pi_t) - P(\pi_t) \;=\;
    \sum_{i=1}^{L} \delta_{e_i}(t)
    \left(\prod_{j=1}^{i-1} \hat{p}_{e_j}(t)\right) S_i(\pi_t).
\end{equation}
The upper clamp gives $\hat p_{e_j}(t) \le 1$, and 
$\delta_{e_i}(t) \ge 0$ on $\mathcal{O}_t$, so dropping the 
prefix product yields 
$\hat{P}_t(\pi_t) - P(\pi_t) \le \sum_i S_i(\pi_t)\,\delta_{e_i}(t)$. 
Combining with the displayed inequality above completes 
the proof.
\end{proof}

\begin{definition}[Prefix Reliability and Path Complexity]
\label{def:path_quantities}
For $\pi = (e_1, \dots, e_L)$ and $i \in \{1, \dots, L\}$, 
the \emph{prefix reliability}: $Q_i(\pi) = \prod_{j=1}^{i-1} p_{e_j} 
    \text{ with } Q_1(\pi) = 1$
is the probability that $e_i$ is reached when $\pi$ is 
attempted. The \emph{path complexity} of $\pi$ is
\begin{equation}
    C(\pi) = \sum_{i=1}^{L} 
    \frac{S_i(\pi)}{\sqrt{Q_i(\pi)}}.
    \label{eq:path_complexity}
\end{equation}
\end{definition}

The prefix and suffix factorize 
$P(\pi) = Q_i(\pi)\,p_{e_i}\,S_i(\pi)$ at every position 
$i$.

\begin{remark}[Interpretation of $C(\pi)$]
\label{rem:complexity_interpretation}
Each term $S_i(\pi)/\sqrt{Q_i(\pi)}$ combines two 
competing position-dependent effects. The factor 
$1/\sqrt{Q_i(\pi)}$ reflects slower concentration of 
$\bar{p}_{e_i}$ for deep edges: $e_i$ is observed only 
when all preceding edges succeed, so its expected 
observation count scales with $Q_i(\pi)$. The factor 
$S_i(\pi)$ captures how a fixed overestimation 
$\delta_{e_i}(t)$ propagates into the path-reliability 
gap (Lemma~\ref{lem:subopt_implies_overest}). The two 
factors pull in opposite directions along $\pi$: early 
edges have $Q_i \approx 1$ but large $S_i$, while late 
edges have small $S_i$ but small $Q_i$. The path 
complexity $C(\pi)$ aggregates these effects into a 
position-aware measure of per-path estimation difficulty.
\end{remark}

Let $M_\pi(t)$ be the number of times $\pi$ has been 
selected by episode $t$. A Chernoff bound (full statement 
in Lemma~\ref{lem:path_obs}, Appendix~\ref{app:path_obs_proof}) 
gives that, with probability $1 - T^{-2}$, the number of 
times edge $e_i$ at position $i$ has been observed within 
those $M_\pi(t)$ selections is $\ge Q_i(\pi)\,M_\pi(t)/2$, 
provided $M_\pi(t) \ge 24 \ln T / Q_i(\pi)$.

\begin{remark}[Horizon in the Analysis]
\label{rem:horizon}
\optinll{} is horizon-free (its confidence radius uses 
$\ln t$), but the analysis below substitutes 
$\ln t \le \ln T$ to obtain uniform static thresholds. 
The looseness for early episodes is absorbed into 
lower-order terms and does not affect the leading 
$\ln T$ rate.
\end{remark}

\begin{lemma}[Per-Path Selection Bound]
\label{lem:per_path_selections}
Under Assumption~\ref{assum:strict_bounds} with $\rho \ge 2$, 
on the joint high-probability event of 
Lemmas~\ref{lem:ucb_optimism},~\ref{lem:ucb_convergence}, 
and~\ref{lem:path_obs}, every suboptimal path 
$\pi \notin \Psi^*$ satisfies
\begin{equation}
    M_\pi(T) \;\le\; \frac{24\,\ln T}{Q_L(\pi)} +
    \frac{8\rho\,C(\pi)^2\,\ln T}{\Delta(\pi)^2}.
\end{equation}
\end{lemma}
\begin{proof}
Fix $\pi = (e_1, \dots, e_L) \notin \Psi^*$ and let $\mathcal{T}_\pi = \{t \le T : \pi_t = \pi\}$, so $M_\pi(T) = |\mathcal{T}_\pi|$. Partition $\mathcal{T}_\pi$ into the warm-up phase $\mathcal{T}_\pi^{\mathrm{wu}} = \{t \in \mathcal{T}_\pi : M_\pi(t) < 24 \ln T / Q_L(\pi)\}$ and the saturation phase $\mathcal{T}_\pi^{\mathrm{sat}} = \mathcal{T}_\pi \setminus \mathcal{T}_\pi^{\mathrm{wu}}$. Since $M_\pi$ increments by $1$ on each $t \in \mathcal{T}_\pi$, the warm-up phase contributes $|\mathcal{T}_\pi^{\mathrm{wu}}| \le 24 \ln T / Q_L(\pi)$.

\textbf{Saturation phase.} For any $t \in \mathcal{T}_\pi^{\mathrm{sat}}$, $M_\pi(t) \ge 24 \ln T / Q_i(\pi)$ for every $i \le L$ (using $Q_i \ge Q_L$), so Lemma~\ref{lem:path_obs} gives $N_{e_i}(t) \ge Q_i(\pi)\,M_\pi(t)/2$. Chaining Lemmas~\ref{lem:subopt_implies_overest} and~\ref{lem:ucb_convergence} (with $\ln t \le \ln T$),
\[
    \Delta(\pi) \;\le\; \sum_{i=1}^{L} S_i(\pi)\,\delta_{e_i}(t)
    \;\le\; \sum_{i=1}^{L} S_i(\pi) \cdot 2\sqrt{\tfrac{2\rho \ln T}{Q_i(\pi)\,M_\pi(t)}}
    \;=\; \tfrac{2\sqrt{2\rho \ln T}}{\sqrt{M_\pi(t)}}\,C(\pi).
\]
Squaring yields $M_\pi(t) \le 8\rho\,C(\pi)^2 \ln T / \Delta(\pi)^2$ for every $t \in \mathcal{T}_\pi^{\mathrm{sat}}$. Evaluating at $t^\star = \max \mathcal{T}_\pi^{\mathrm{sat}}$ (or noting $|\mathcal{T}_\pi^{\mathrm{sat}}| = 0$ otherwise) gives $|\mathcal{T}_\pi^{\mathrm{sat}}| \le M_\pi(t^\star) \le 8\rho\,C(\pi)^2 \ln T / \Delta(\pi)^2$.

Combining: $M_\pi(T) \le 24 \ln T / Q_L(\pi) + 8\rho\,C(\pi)^2 \ln T / \Delta(\pi)^2$.
\end{proof}

\begin{theorem}[Path-Dependent Regret of \optinll{}]
\label{thm:path_dependent_regret}
Under Assumption~\ref{assum:strict_bounds} with 
$\rho \ge 2$, the expected cumulative regret of \optinll{} 
satisfies
\begin{equation}
    \mathbb{E}[R(T)] \;\le\;
    \sum_{\pi \notin \Psi^*}
    \frac{8\rho\,C(\pi)^2\,\ln T}{\Delta(\pi)} \;+\;
    \sum_{\pi \notin \Psi^*}
    \frac{24\,\ln T}{Q_L(\pi)} \;+\;
    O\!\left(|E| + \frac{|\Psi(v_s,v_g)|\,L_{\max}}{T}\right).
    \label{eq:path_dep_bound}
\end{equation}
\end{theorem}
\begin{proof}[Proof Sketch]
On the joint high-probability event of 
Lemmas~\ref{lem:ucb_optimism},~\ref{lem:ucb_convergence}, 
and~\ref{lem:path_obs}, Lemma~\ref{lem:per_path_selections} 
bounds $M_\pi(T)$ for every suboptimal $\pi$. The 
contribution of each $\pi$ to expected regret is 
$\Delta(\pi)\,M_\pi(T)$; summing and absorbing the 
warm-up term via $\Delta(\pi) \le 1$ gives the leading 
two sums. Off-event contributions are $O(|E|)$ from 
optimism/UCB failures plus 
$O(|\Psi(v_s, v_g)| L_{\max}/T)$ from observation 
failures. Full proof in Appendix~\ref{app:thm_proof}.
\end{proof}

\subsection{Discussion}
\label{sec:theory_discussion}

\textbf{Comparison with the CCB bound.}
The edge-level CombCascade 
bound~\eqref{eq:combcascade_bound} sums 
$O(1/\Delta_{e,\min})$ over suboptimal edges and scales 
globally with $1/p^*$, avoiding any enumeration of paths. 
In contrast, the path-level bound sums 
$O(C(\pi)^2 \ln T / \Delta(\pi))$ over suboptimal paths, 
with $C(\pi)$ capturing position-dependent structure 
discarded by the edge-level analysis 
(Remark~\ref{rem:complexity_interpretation}). Neither 
dominates: the edge-level bound is preferable when 
$|\Psi(v_s, v_g)|$ is large relative to $|E|$ (typically 
dense graphs), while the path-level bound is more 
informative when the path count is polynomial in $|V|$ 
(layered DAGs, sparse networks). Both bounds depend on 
unknown problem parameters and serve as theoretical 
guarantees rather than computable predictions; both are 
simultaneously valid, so the minimum applies. The 
path-level bound also carries the 
$|\Psi(v_s,v_g)|\,L_{\max}/T$ failure-event term, which 
vanishes asymptotically but can dominate at moderate $T$ 
in dense regimes.

\textbf{Absence of guarantees for \bayesnll{}.}
The analysis above applies only to \optinll{} 
(Remark~\ref{rem:ts_regret}). 
Section~\ref{sec:experiments} exhibits an adversarial 
instance on which \bayesnll{} fails to converge over 
$T = 10{,}000$ episodes---behavior consistent with 
the $\Omega(2^{k^*})$ obstruction 
of~\cite{wang2018thompson}.
\bayesnll{} typically 
outperforms \optinll{} in our experiments, but its 
performance can degrade sharply when the optimal path is 
long and suboptimal competitors are edge-disjoint---the 
regime where \optinll{}'s position-aware analysis remains 
tight.

Further discussions regarding Variance-awareness and warm-up cost appear in 
Appendix~\ref{app:theory_discussion}.

\section{Experimental Evaluation}
\label{sec:experiments}

We empirically confirm Theorem~\ref{thm:path_dependent_regret} 
and evaluate practical performance through three 
experiments: regret-curve validation 
(Section~\ref{sec:exp1}), comprehensive baseline 
comparison (Section~\ref{sec:exp2}), and scaling analysis 
(Section~\ref{sec:exp3}). All results are averaged over 10 random topologies per 
setting $\times$ 20 seeds per topology (200 runs). We consider the following four domains:

\textbf{1. Erd\H{o}s--R\'enyi random 
graphs}~\cite{erdos_renyi} serve as the primary validation 
domain: directed graphs with edge probability $0.4$, 
reliabilities $p_e \sim \text{Uniform}[0.01, 0.99]$, and 
$(v_s, v_g)$ sampled uniformly with $v_s \ne v_g$.
\textbf{2. Layered DAGs} provide a controlled 
path-diversity benchmark with $D$ layers and width $w$ 
($v_s$ alone in layer 1, $v_g$ alone in layer $D$, full 
connectivity between adjacent layers), giving 
$|V| = w(D-2) + 2$ and $|E| = 2w + w^2(D-3)$. 
\textbf{3. Grid worlds} are $n \times n$ grids with 
4-connectivity, $|V| = n^2$, $|E| = 4n(n-1)$, $v_s$ at the 
top-left and $v_g$ at the bottom-right; reliabilities 
uniform on $[0.01, 0.99]$.
\textbf{4. Quantum repeater networks} use a 25-node, 43-edge 
abstraction of the SURFnet 
backbone~\citep{knight2011internet}, with edge 
reliabilities following the standard fiber-attenuation model~\citep{Wehner2018}
and endpoints sampled at topological distance $\ge 3$. 
Full details in Appendix~\ref{app:surfnet}.


\textbf{Convergence:} An algorithm has converged at the 
first episode $t$ at which its greedy policy (the path 
extracted from current estimates without exploration 
bonuses; details in Appendix~\ref{app:greedy_extraction}) 
equals some $\pi^* \in \Psi^*$ for 10 consecutive 
episodes. The \emph{convergence rate} is the fraction of 
(topology, seed) runs that converged before horizon $T$.
\textbf{Tie-breaking:} Exact floating-point ties on path 
costs are broken by adding independent 
$\mathrm{Uniform}(0, 10^{-8})$ noise per edge---negligible 
relative to typical log edge costs.
\textbf{Hyperparameters:} \optinll{} uses $\rho = 2.0$ in 
Experiment~1 (matching 
Theorem~\ref{thm:path_dependent_regret}) and $\rho = 1.0$ 
in Experiments~2 and~3 (selected by a preliminary sweep 
on Erd\H{o}s--R\'enyi optimizing median convergence 
episode). \bayesnll{} uses the default 
$\mathrm{Beta}(1, 1)$ prior.
\textbf{Baselines.}
We compare against six baselines spanning CCB, SSP 
planning, RL, and ablations as listed below. Full implementation details are provided in Appendix~\ref{app:baselines}.

\textbf{1. CombCascade}~\cite{kveton2015combinatorial} is the 
canonical CCB algorithm; we use Dijkstra on 
$-\log \hat p_e$ to implement its maximization oracle 
(implementation differences from \optinll{} in 
Appendix~\ref{app:combcascade_diffs}).
\textbf{2. RTDP}~\cite{Barto1993rtdp} is the standard SSP 
baseline. \textbf{3. LRTDP}~\citep{Bonet2003lrtdp} extends RTDP 
with state-level convergence labeling.
\textbf{4. Q-Learning} is a standard tabular RL baseline 
($\alpha = 0.1$, $\gamma = 0.99$, $\varepsilon$-greedy 
with decaying $\varepsilon$, optimistic hop-count 
initialization).
\textbf{5. CUCB} is a log-transform ablation: identical to 
\optinll{} but runs Dijkstra on 
$w(e) = 1 - \hat p_e(t)$ instead of $-\log \hat p_e(t)$, 
optimizing the wrong objective.
\textbf{6. Random} selects paths via random walks with cycle 
avoidance and dead-end restart, providing a worst-case 
floor.

\subsection{Experiment 1: Regret Curve Validation}
\label{sec:exp1}

We test the $O(\ln T)$ scaling predicted by 
Theorem~\ref{thm:path_dependent_regret} and compare the 
two bounds against empirical regret over $T = 200{,}000$ 
episodes.

\textbf{$O(\ln T)$ Scaling.}
We fit $\mathbb{E}[R(T)] \approx a \ln T + b$ over 40 
log-spaced episodes for each (algorithm, domain) pair 
(Figure~\ref{fig:regret_curves}; full per-domain results 
in Appendix~\ref{app:exp1_table}). \optinll{}, 
CombCascade, and \bayesnll{} all achieve 
$R^2 \ge 0.96$ across the four domains, 
consistent with $O(\ln T)$ regret 
(Theorem~\ref{thm:path_dependent_regret} and 
Eq.~\ref{eq:combcascade_bound}); slope ranges are 
$6$--$17$ for \bayesnll{} (an order of magnitude smaller 
than the UCB-based methods), $93$--$168$ for CombCascade, 
and $118$--$227$ for \optinll{}. By contrast, RTDP, LRTDP, and Random have $R^2 \approx 0.75$, 
suggesting their regret grows approximately linearly 
in $T$ rather than logarithmically. 
We caution that high log-linear $R^2$ is suggestive 
rather than conclusive evidence of $O(\ln T)$, since 
slowly-growing power-law curves can also fit. Large 
standard deviations on $R(T)$ are observed, reflecting variation across 
random topologies and stochastic edge outcomes.

\begin{figure}[t]
\caption{Empirical regret of all algorithms across four domains over $T = 200{,}000$ episodes.}
\label{fig:regret_curves}
\centering
\begin{subfigure}[b]{0.25\textwidth}
    \centering
    \includegraphics[width=\textwidth, trim=0 0 0 0.5cm, clip]{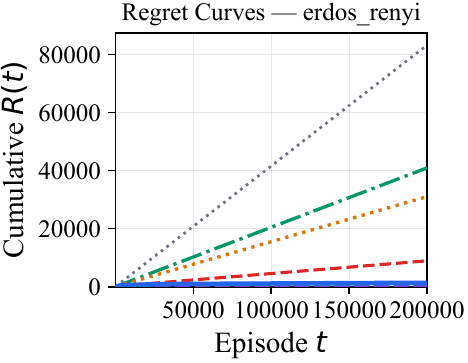}
    \caption{Erd\H{o}s--R\'{e}nyi}
    \label{fig:regret_erdos_renyi}
\end{subfigure}
\hfill
\begin{subfigure}[b]{0.24\textwidth}
    \centering
    \includegraphics[width=\textwidth, trim=5mm 0 0 0.5cm, clip]{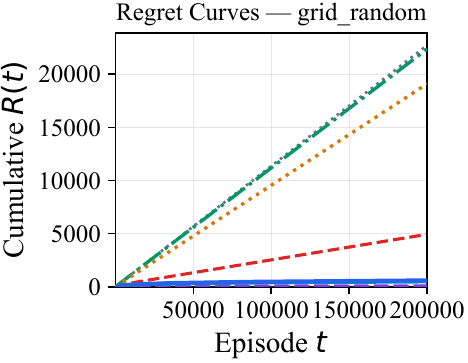}
    \caption{Grid World}
    \label{fig:regret_grid_random}
\end{subfigure}
\hfill
\begin{subfigure}[b]{0.24\textwidth}
    \centering
    \includegraphics[width=\textwidth, trim=5mm 0 0 0.5cm, clip]{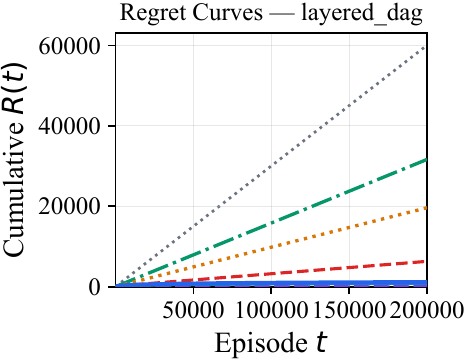}
    \caption{Layered DAG}
    \label{fig:regret_layered_dag}
\end{subfigure}
\hfill
\begin{subfigure}[b]{0.24\textwidth}
    \centering
    \includegraphics[width=\textwidth, trim=5mm 0 0 0.5cm, clip]{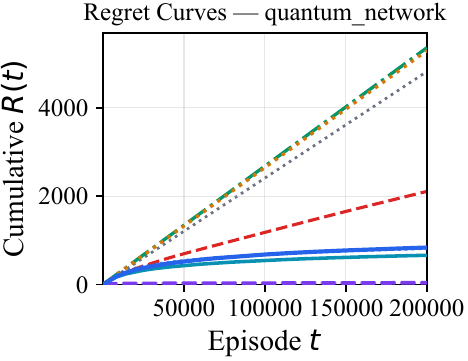}
    \caption{Quantum Network}
    \label{fig:regret_quantum_network}
\end{subfigure}

\makebox[\textwidth][c]{%
    \includegraphics[width=1.9\textwidth]{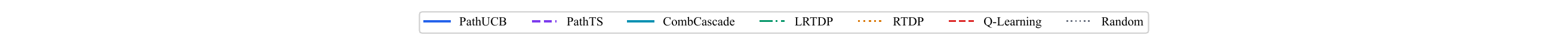}}
    \vspace{-2em}
\end{figure}

\textbf{Bound Comparison.}
We empirically compare the edge-level CombCascade bound 
(Eq.~\ref{eq:combcascade_bound}) and the path-level 
\optinll{} bound (Theorem~\ref{thm:path_dependent_regret}) 
against \optinll{}'s empirical regret 
(Figure~\ref{fig:bounds_comparison} in 
Appendix~\ref{app:bound_comparison}). Two observations 
align with Section~\ref{sec:theory_discussion}: (i) 
neither bound is uniformly tighter---the path-level bound 
dominates on the sparse Layered DAG and Grid domains 
while the edge-level bound dominates on the dense 
Erd\H{o}s--R\'enyi and Quantum networks, matching the 
$|\Psi(v_s, v_g)|$-vs-$|E|$ tradeoff; and (ii) both 
bounds are loose which is consistent with the 
worst-case nature of the gap-dependent constants. Nonetheless,   the logarithmic regret accumulation guarantee (at the limit) is still valid and empirically observed over the $2 \times 10^5$ episodes.

\subsection{Experiment 2: Comprehensive Baseline Comparison}
\label{sec:exp2}

We evaluate all agents on the four primary domains over 
$T = 10{,}000$ episodes 
(Figure~\ref{fig:sr_baselines}---bottom-right is best), 
then introduce a fifth adversarial domain, \emph{Path 
Trap}, to stress-test \bayesnll{} on the structure 
underlying Remark~\ref{rem:ts_regret}.

\begin{figure}[t]
\centering
\caption{Empirical regret $R(T) \pm 95\%$ CI at $T = 10{,}000$ 
episodes, (y-axis) versus convergence rate (x-axis) for each 
algorithm across the four primary domains with domain size ($|V|$) in parenthesis. \textbf{Bottom-right is better}.}
\begin{subfigure}[b]{0.26\textwidth}
    \centering
    \includegraphics[width=\textwidth, trim=0 0 0 4mm, clip]{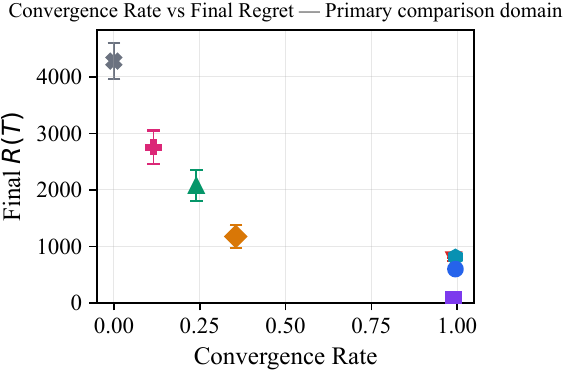}
    \caption{Erd\H{o}s-R\'{e}nyi (13)}
    \label{fig:sr_erdos_renyi}
\end{subfigure}
\hfill
\begin{subfigure}[b]{0.23\textwidth}
    \centering
    \includegraphics[width=\textwidth, trim=6mm 0 0 4mm, clip]{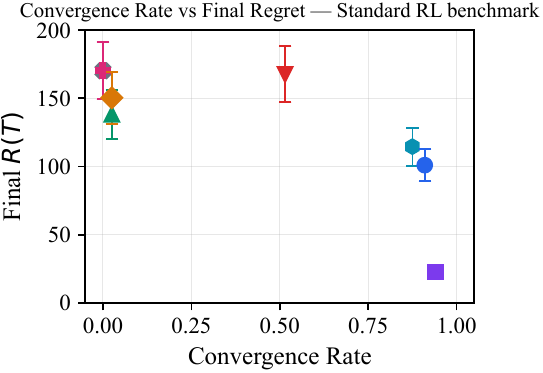}
    \caption{Grid World (36)}
    \label{fig:sr_grid_random}
\end{subfigure}
\hfill
\begin{subfigure}[b]{0.21\textwidth}
    \centering
    \includegraphics[width=\textwidth, trim=6mm 0 0 4mm, clip]{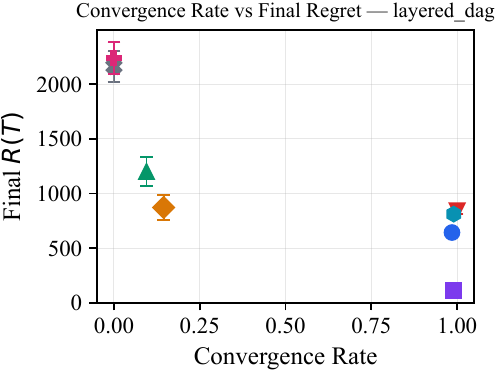}
    \caption{Layered DAG (17)}
    \label{fig:sr_layered_dag}
\end{subfigure}
\hfill
\begin{subfigure}[b]{0.24\textwidth}
    \centering
    \includegraphics[width=\textwidth, trim=6mm 0 0 4mm, clip]{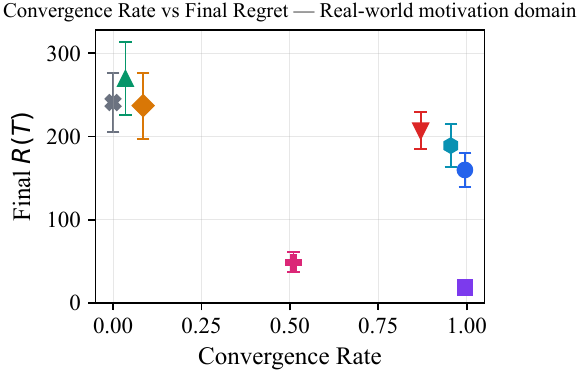}
    \caption{Quantum Network (25)}
    \label{fig:sr_quantum}
\end{subfigure}
\makebox[\textwidth][c]{%
    \includegraphics[width=1.8\textwidth]{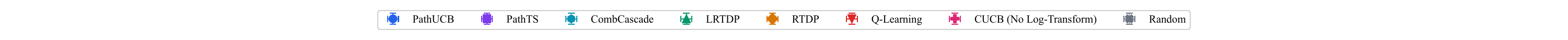}
}
\vspace{-2em}
\label{fig:sr_baselines}
\end{figure}

\textbf{Primary Results.}
\bayesnll{} achieves the lowest final regret on all four 
domains by a substantial margin---$4\times$ to 
$10\times$ lower than \optinll{} and CombCascade. The 
three bandit-style algorithms converge to an optimal 
policy in $\ge 85\%$ of runs, while RTDP and LRTDP 
converge in only $0$--$46\%$, confirming that SSP solvers 
are ill-suited to \srp{}: the closed-loop machinery they 
invoke is unnecessary under open-loop equivalence 
(Lemma~\ref{lemma:open_loop}) and does not exploit the 
combinatorial-bandit structure. LRTDP underperforms RTDP 
on every domain, plausibly because its convergence 
labeling locks in optimistic bias that RTDP eventually 
corrects via continued 
backups~\cite{mcmahan2005bounded}---an effect amplified 
by \srp{}'s severe global-reset penalty.
\textbf{Log-Transform Ablation.}
CUCB underperforms \optinll{} on every domain, incurring 
$1.2\times$--$2.1\times$ higher final regret and  lower convergence rates (e.g., $47\%$ vs. 
$89\%$ on Grid World). The gap confirms the necessity of 
the log-transform (Lemma~\ref{lem:structural}): CUCB 
biases path selection toward edge-rich, 
low-individual-failure routes rather than 
multiplicatively reliable ones.
\begin{wrapfigure}{r}{0.39\textwidth}
\caption{Path Trap topology.}
\centering
\includegraphics[width=0.38\textwidth, trim=0 0 0 5mm, clip]{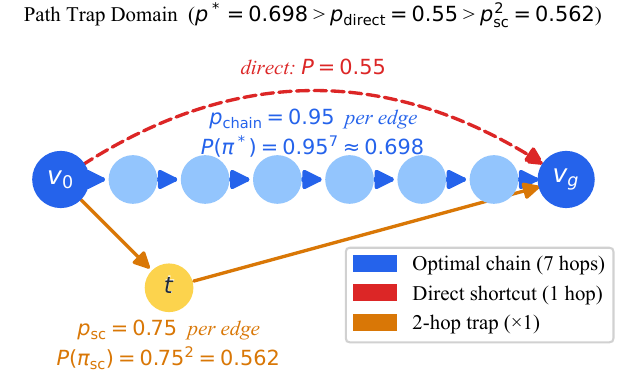}
\label{fig:path_trap_graph}
\vspace{-1em}
\end{wrapfigure}

\textbf{Exploration Failure of Thompson Sampling.} \emph{Path Trap} (Figure~\ref{fig:path_trap_graph}) is a 
deterministic-parameter graph with a 7-hop optimal chain 
($p_{\text{chain}} = 0.95$, $P(\pi^*) \approx 0.698$), 
a 2-hop shortcut ($p_{\text{shortcut}} = 0.75$, 
$P \approx 0.563$ each), and one direct edge 
($p_{\text{direct}} = 0.55$). Edge probabilities are 
fixed; outcomes remain stochastic. The construction 
penalizes algorithms that over-commit to short paths with 
frequent positive feedback rather than explore the longer 
but more reliable chain---the failure mode of 
combinatorial Thompson Sampling on multiplicative-reward 
instances~\cite{wang2018thompson}.
\begin{wrapfigure}{r}{0.36\textwidth}
\centering
\vspace{-1em}
\caption{Regret vs.\ convergence on the Path 
Trap domain. Markers as in Figure~\ref{fig:sr_baselines}.}
\includegraphics[width=0.34\textwidth, trim=0 0 0 5mm, clip]{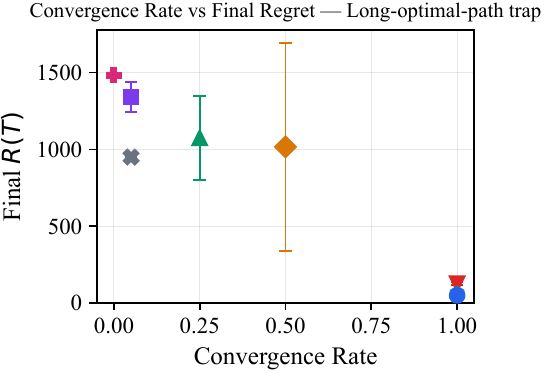}
\label{fig:sr_path_trap}
\vspace{-1em}
\end{wrapfigure}
As illustrated in Figure~\ref{fig:sr_path_trap},
\bayesnll{}'s regret reaches $R(T) = 1{,}421$---worse than Random ($986$, which selects the optimal chain with probability $1/3$)---while \optinll{} 
achieves the best result, $R(T) = 66$ with $100\%$ 
convergence. The asymmetry reflects the difference between 
UCB optimism and Thompson Sampling on this instance: the 
shortcuts generate frequent positive feedback that rapidly 
concentrates \bayesnll{}'s Beta posteriors around the 
suboptimal alternatives, while \optinll{} maintains 
inflated estimates for the under-explored chain edges via 
its UCB confidence radius and eventually routes along the 
optimal path.
This is consistent with the $\Omega(2^{k^*})$ obstruction 
of~\cite{wang2018thompson} (Theorem~3) for combinatorial 
Thompson Sampling on multiplicative-reward instances. With 
$k^* = 7$, Path Trap sits in the regime where this 
obstruction predicts difficulty within practical horizons. 
The practical takeaway is not that \bayesnll{} should be 
avoided---it remains the strongest performer on every 
other domain we tested---but that adversarial instances 
exist on which posterior-sampling can fail dramatically, 
and \optinll{}'s theoretical guarantee provides the safer 
fallback.

\subsection{Experiment 3: Scaling Analysis}
\label{sec:exp3}
A scaling study (full results in 
Appendix~\ref{app:exp3}) confirms two trends. First, on 
both Erd\H{o}s--R\'enyi (where $|\Psi(v_s, v_g)|$ grows 
super-polynomially in $|V|$) and Layered DAG (where it 
grows polynomially in width $w$), \optinll{} and 
\bayesnll{} exhibit approximately linear regret growth 
in $|E|$, indicating edge-count scalability; in dense 
graphs $|E| = O(|V|^2)$, so this translates to 
polynomial-in-$|V|$ scaling. Second, on Erd\H{o}s--R\'enyi 
with $|V| = 15$, regret grows super-linearly as the 
support span $p_{\max} - p_{\min}$ shrinks, consistent 
with the gap-dependent factor $1/\Delta(\pi)$ in 
Theorem~\ref{thm:path_dependent_regret}: as edge 
reliabilities concentrate, path reliabilities cluster and 
the per-path gaps $\Delta(\pi)$ shrink approximately with 
the span.

\section{Conclusion}
\label{sec:conclusion}

We introduced \emph{Stochastic Reset Pathfinding} (\srp{}) 
and showed that its global-reset structure makes the 
optimal policy open-loop, placing the problem within the 
combinatorial cascading bandit framework. Building on this 
reduction, we developed a path-level regret bound for 
\optinll{} via a per-path complexity $C(\pi)$ that captures 
how each edge's prefix and suffix reliability jointly 
determine its contribution. Our experiments validate the 
$O(\ln T)$ rate, demonstrate the strong empirical 
performance of the Thompson Sampling instantiation 
\bayesnll{}, and exhibit an adversarial instance on which 
\bayesnll{} fails---consistent with a known exponential 
obstruction. Natural future directions include a 
variance-aware (Bernstein) extension of the path-level 
analysis, a regret bound for \bayesnll{} bridging existing 
matroid- and semi-bandit-style results, and adaptation of 
CombCascade's prefix argument to close the warm-up gap in 
$1/Q_L(\pi)$.

\bibliographystyle{plain}
\bibliography{ref.bib}

\appendix
\newpage

\section{Theoretical Status of \bayesnll{}: Extended Discussion}
\label{app:ts_discussion}

The two prior analyses closest to \bayesnll{} are those of 
Cheung et al.~\cite{zhong2021thompson} and Wang \& 
Chen~\cite{wang2018thompson}. Neither directly extends to 
\srp{}.
Cheung et al.\ prove an $O(L \log T / \Delta)$ regret bound 
for Beta-Bernoulli Thompson Sampling on cascading bandits, 
but their analysis relies on the uniform matroid structure 
of the top-K feasible set---specifically, exchangeability 
of items in the optimal solution---which fails for general 
graph paths.

Wang \& Chen analyze combinatorial Thompson Sampling for 
general CMAB with multiplicative rewards, but under 
semi-bandit rather than cascade feedback. Adapting their 
concentration arguments to cascade observations introduces 
position-dependent factors that, in our analysis, take the 
form of the prefix reliabilities $Q_i(\pi)$ governing 
per-edge observation probabilities 
(Definition~\ref{def:path_quantities} and the warm-up term 
in Theorem~\ref{thm:path_dependent_regret}). Their 
Theorem~3 also exhibits an instance with conjunctive 
product reward on which any Beta-Bernoulli TS variant 
provably incurs $\Omega(2^{k^*})$ regret, where 
$k^*$ is the size of the optimal solution.

Together, these results suggest that a regret analysis for 
\bayesnll{} is within reach but would require combining 
\cite{wang2018thompson}'s general-feasible-set techniques 
with \cite{zhong2021thompson}'s cascade-feedback handling, 
and may inherit a path-length-exponential leading constant. 
\bayesnll{} is presented as a practical alternative 
motivated by the strong empirical performance of Thompson 
Sampling in combinatorial 
settings~\cite{chapelle2011empirical, wang2018thompson, 
zhong2021thompson}; see Section~\ref{sec:exp2}.

\section{Implementation Differences Between \optinll{} and CombCascade}
\label{app:combcascade_diffs}

\optinll{} and CombCascade share the same algorithmic 
template (per-edge UCB, log-transformed shortest path 
selection) but differ in three implementation choices:

\textbf{Initialization.} \optinll{} sets $\hat p_e(t) = 1$ 
for unvisited edges, ensuring strict optimism on the first 
visit. CombCascade assumes an initial sweep that draws a 
Bernoulli outcome for every edge before the main loop, 
guaranteeing $N_e(t) \ge 1$ throughout.

\textbf{Confidence radius.} \optinll{} uses 
$\sqrt{\rho \ln t / N_e(t)}$ with $\rho \ge 2$. CombCascade 
uses $\sqrt{1.5 \ln(t-1)/T_{t-1}(e)}$. Both yield $t^{-3}$ 
tail probabilities under Hoeffding's inequality.

\textbf{Clamping.} \optinll{} clips two-sided to 
$[p_{\min}, 1]$. CombCascade clips only above at $1$ (no 
lower clamp).

The strict optimism and the upper clamp at $1$ are used 
in Lemma~\ref{lem:subopt_implies_overest} to bound the 
prefix product and isolate the suffix-weighted 
overestimation sum that drives the path-level analysis. The 
lower clamp at $p_{\min}$ keeps log-weights finite 
throughout learning. Apart from CombCascade's initial 
sweep, the two algorithms perform comparably in our 
experiments (Section~\ref{sec:experiments}).

\section{Concentration Lemmas for the Clamped UCB Estimator}
\label{app:concentration}

Both Lemmas~\ref{lem:ucb_optimism} 
and~\ref{lem:ucb_convergence} reduce to a single 
two-sided concentration claim for the empirical mean 
$\bar p_e(t)$.

\begin{lemma}[Two-Sided Concentration]
\label{lem:two_sided_concentration}
Under $\rho \ge 2$, for each $e \in E$ and $t \ge 1$,
\begin{equation}
    \Pr\!\left(\, |\bar{p}_e(t) - p_e| \le 
    \sqrt{\tfrac{\rho \ln t}{N_e(t)}} \;\Big|\; 
    N_e(t) \ge 1 \right) \;\ge\; 1 - 2t^{-3}.
\end{equation}
\end{lemma}
\begin{proof}
By the fictitious sampling 
argument~\citep{lattimore2020bandit}, couple the observed 
outcomes of edge $e$ to a fixed i.i.d.\ sequence 
$Y_1, Y_2, \dots \sim \mathrm{Bernoulli}(p_e)$, so 
$\bar{p}_e(t) = \tfrac{1}{N_e(t)}\sum_{j=1}^{N_e(t)} Y_j$. 
Since $N_e(t) \in \{1, \dots, t\}$ is data-dependent, a 
union bound combined with Hoeffding's inequality gives
\[
\Pr\!\left(\exists\, n \in \{1,\dots,t\}\,:\, 
\bigl|\tfrac{1}{n}\!\sum_{j=1}^{n}\!Y_j - p_e\bigr| 
> \sqrt{\tfrac{\rho \ln t}{n}}\right)
\le 2\!\sum_{n=1}^{t} e^{-2\rho \ln t} 
= 2t^{1-2\rho},
\]
which is at most $2t^{-3}$ for $\rho \ge 2$.
\end{proof}

\begin{lemma}[UCB Optimism]
\label{lem:ucb_optimism}
Under Assumption~\ref{assum:strict_bounds} with $\rho \ge 2$, 
$\Pr(\mathcal{O}_t) \ge 1 -|E|\,t^{-3}$ for all $t \ge 1$. 
Consequently, the cumulative expected regret on episodes 
where $\mathcal{O}_t^c$ holds is $O(|E|)$, independent of 
$T$.
\end{lemma}
\begin{proof}[Proof of Lemma~\ref{lem:ucb_optimism}]
Fix $e \in E$ and $t \ge 1$. If $N_e(t) = 0$, then 
$\hat{p}_e(t) = 1 > p_e$ deterministically. If 
$N_e(t) \ge 1$, the event 
$\{\hat{p}_e(t) < p_e\}$ implies 
$\bar{p}_e(t) - p_e < -\sqrt{\rho \ln t / N_e(t)}$, which 
has probability at most $t^{-3}$ by the lower-tail half of 
Lemma~\ref{lem:two_sided_concentration}. A union bound over 
$|E|$ edges yields $\Pr(\mathcal{O}_t^c) \le |E|\,t^{-3}$. 
The cumulative-regret claim follows from $r(\pi_t) \le 1$ 
and $\sum_{t=1}^{\infty} t^{-3} < \infty$.
\end{proof}

\begin{lemma}[UCB Convergence Rate]
\label{lem:ucb_convergence}
Under Assumption~\ref{assum:strict_bounds} with $\rho \ge 2$, 
for any edge $e \in E$ and any episode $t$ with 
$N_e(t) > 0$,
\begin{equation}
    \delta_e(t) \;\le\; 
    2\sqrt{\frac{\rho \ln t}{N_e(t)}}
\end{equation}
with probability at least $1 - t^{-3}$.
\end{lemma}
\begin{proof}[Proof of Lemma~\ref{lem:ucb_convergence}]
Fix $e \in E$ and $t$ with $N_e(t) \ge 1$. Let 
$U_e(t) = \bar{p}_e(t) + \sqrt{\rho \ln t / N_e(t)}$ be 
the unclamped UCB, so 
$\hat{p}_e(t) = \mathrm{clip}(U_e(t), p_{\min}, 1)$. 
Clamping at $p_{\min}$ only decreases $\hat p_e(t)$ below 
$p_e$ (giving $\delta_e(t) \le 0$); clamping at $1$ gives 
$\delta_e(t) = 1 - p_e \le U_e(t) - p_e$ since $U_e(t) \ge 1$ 
in that case; the unclamped case gives 
$\delta_e(t) = U_e(t) - p_e$ directly. In all cases,
\[
\delta_e(t) \le U_e(t) - p_e 
= (\bar{p}_e(t) - p_e) + \sqrt{\rho \ln t / N_e(t)}.
\]
By the upper-tail half of 
Lemma~\ref{lem:two_sided_concentration}, 
$\bar{p}_e(t) - p_e \le \sqrt{\rho \ln t / N_e(t)}$ with 
probability at least $1 - t^{-3}$, giving 
$\delta_e(t) \le 2\sqrt{\rho \ln t / N_e(t)}$.
\end{proof}

\section{Path-Conditional Observation Bound 
(Lemma~\ref{lem:path_obs})}
\label{app:path_obs_proof}

For brevity, let $N_{e_i}^\pi(t)$ denote the number of 
selections of $\pi$ within the first $t$ episodes in 
which edge $e_i$ at position $i$ was observed.

\begin{lemma}[Path-Conditional Observation Bound]
\label{lem:path_obs}
Under Assumption~\ref{assum:strict_bounds}, fix a path 
$\pi$ and position $i \in \{1, \dots, |\pi|\}$. With 
probability at least $1 - T^{-2}$,
\begin{equation}
    N_{e_i}^\pi(t) \;\ge\; \frac{Q_i(\pi)}{2}\,M_\pi(t)
    \qquad \text{for all } t \le T \text{ with } 
    M_\pi(t) \ge \frac{24 \ln T}{Q_i(\pi)}.
\end{equation}
\end{lemma}
\begin{proof}
For each edge $e \in \pi$, let 
$Y^e_1, Y^e_2, \dots \sim \mathrm{Bernoulli}(p_e)$ be the 
exogenous fictitious-sampling 
sequence~\citep{lattimore2020bandit}, with each $Y^e_n$ 
giving the outcome of $e$ on its $n$-th attempt across the 
entire history. On the $j$-th episode in which $\pi$ is 
selected, the observation indicator $I_j$---which equals 
$1$ iff edges $e_1, \dots, e_{i-1}$ all succeed on that 
episode---is determined by reading the next unused outcome 
from each of the $i-1$ prefix sequences. Because the prefix 
edges' fictitious sequences are mutually independent, $I_j$ 
is, conditionally on the past, a Bernoulli draw with 
parameter $Q_i(\pi) = \prod_{k < i} p_{e_k}$.

The counter $M_\pi(t)$ is data-dependent: past outcomes 
influence UCB estimates and hence subsequent path 
selections. We handle this via a union bound over the 
possible values of $M_\pi(t)$. For each fixed $m \in 
\{1, \dots, T\}$, the partial sum $\sum_{j=1}^m I_j$ 
stochastically dominates a $\mathrm{Binomial}(m, Q_i(\pi))$ 
random variable along the coupled sequences, so the 
multiplicative Chernoff bound with $\delta = 1/2$ gives
\begin{equation}
    \Pr\!\left(\sum_{j=1}^{m} I_j < 
    \frac{Q_i(\pi)\,m}{2}\right)
    \;\le\; \exp\!\left(-\frac{Q_i(\pi)\,m}{8}\right) 
    \;\le\; T^{-3}
\end{equation}
whenever $m \ge 24 \ln T / Q_i(\pi)$. A union bound over 
the at most $T$ such values of $m$ gives total failure 
probability at most $T^{-2}$. On the complementary event, 
$N_{e_i}^\pi(t) = \sum_{j=1}^{M_\pi(t)} I_j 
\ge Q_i(\pi)\,M_\pi(t)/2$ for every $t \le T$ with 
$M_\pi(t) \ge 24 \ln T / Q_i(\pi)$.
\end{proof}

\section{Path-Dependent Regret of \optinll{} 
(Theorem~\ref{thm:path_dependent_regret})}
\label{app:thm_proof}

\begin{proof}[Proof of Theorem~\ref{thm:path_dependent_regret}]
Let $\mathcal{G}$ denote the joint event on which 
Lemmas~\ref{lem:ucb_optimism},~\ref{lem:ucb_convergence}, 
and~\ref{lem:path_obs} all hold for every $t \le T$, edge, 
path, and position. We split 
$\mathbb{E}[R(T)] = \mathbb{E}[R(T)\mathbf{1}_\mathcal{G}] + 
\mathbb{E}[R(T)\mathbf{1}_{\mathcal{G}^c}]$.

\textbf{Off $\mathcal{G}$.} Optimism failures and per-edge 
UCB convergence failures each contribute $O(|E|)$ via 
$\sum_t t^{-3} < \infty$. Path-conditional observation 
failures occur with probability $T^{-2}$ per 
(path, position) pair; summing over at most 
$|\Psi(v_s, v_g)| L_{\max}$ pairs and multiplying by the 
maximum regret $T$ gives 
$O(|\Psi(v_s, v_g)| L_{\max}/T)$.

\textbf{On $\mathcal{G}$.} Optimal paths incur zero regret, 
and Lemma~\ref{lem:per_path_selections} bounds $M_\pi(T)$ 
for each $\pi \notin \Psi^*$. Multiplying by $\Delta(\pi)$, 
distributing, and using $\Delta(\pi) \le 1$ to absorb the 
warm-up term:
\begin{equation}
    \mathbb{E}[R(T)\mathbf{1}_\mathcal{G}] 
    \;\le\; \sum_{\pi \notin \Psi^*} \Delta(\pi)\,M_\pi(T) 
    \;\le\; \sum_{\pi \notin \Psi^*}
    \frac{8\rho\,C(\pi)^2 \ln T}{\Delta(\pi)} +
    \sum_{\pi \notin \Psi^*} \frac{24 \ln T}{Q_L(\pi)}.
\end{equation}
Adding the off-$\mathcal{G}$ contribution gives the 
theorem.
\end{proof}

\section{Further Theoretical Discussion}
\label{app:theory_discussion}

\paragraph{Variance awareness.}
Vial et al.~\cite{vial2022minimax} show that 
Hoeffding-based UCB is order-wise suboptimal for cascading 
bandits in the problem-independent regret, with 
variance-aware (Bernstein) confidence sets attaining the 
minimax lower bound. \optinll{} inherits this suboptimality 
on \srp{} instances with high-reliability paths. We use 
Hoeffding bounds to keep the path-level decomposition in 
Lemma~\ref{lem:subopt_implies_overest} self-contained; 
extending to Bernstein would augment $C(\pi)$ with edge 
variance terms in place of $1/\sqrt{Q_i(\pi)}$, recovering 
the variance-aware rates at the path level. We leave this 
as future work.

\paragraph{Warm-up cost.}
The term $\sum_\pi 24 \ln T / Q_L(\pi)$ counts the 
selections of each suboptimal path before 
Lemma~\ref{lem:path_obs} takes effect at the deepest 
position. Since $Q_L(\pi) \ge p_{\min}^{L-1}$, this term 
can grow exponentially in $L$ and dominate the bound on 
long paths. CombCascade sidesteps this via a prefix 
argument (Lemma~1 of~\cite{kveton2015combinatorial}): a 
suboptimal path can be distinguished from the optimum on 
a prefix alone. Adapting this to the path level---where 
$C(\pi)$ is defined over the full path---is a natural 
future direction.

\section{Quantum Repeater Network: SURFnet Abstraction}
\label{app:surfnet}

The 25-node abstraction of the SURFnet backbone provides a 
pragmatic baseline for near-term regional quantum-internet 
deployments~\citep{knight2011internet, 
chakraborty2020entanglement}. Nominal link lengths 
$L_{\text{nom}} \in [15, 100]\,\text{km}$ adhere to 
physical repeater spacing constraints; long-haul routes 
exceeding a single quantum hop (e.g., the approximately 
175 km Amsterdam--Groningen route) are logically 
segmented via 
intermediate hubs (Almere, Zwolle). Bidirectionality is 
enforced ($p_e$ identical in both directions per link) 
since each link uses a single physical fiber.

\textbf{Edge reliability.} For each instance, an effective 
length $L_{\text{eff}} = \max(1, L_{\text{nom}} + \eta_e)$ 
is generated with $\eta_e \sim \mathcal{N}(0, 4)$ km 
(modeling environmental drift and link-characterisation 
uncertainty). The success probability follows the standard 
fiber-attenuation model
\begin{equation}
    p_e = \exp(-\kappa L_{\text{eff}}), 
    \qquad \kappa = 0.02\,\text{km}^{-1},
\end{equation}
clipped to $[p_{\min}, 1]$ with 
$p_{\min} = 0.01$. The chosen $\kappa$ is optimistic 
relative to the standard $\kappa = 
0.046\,\text{km}^{-1}$~\citep{Wehner2018}, modeling 
near-future low-loss fiber or an effective transmission 
probability that absorbs multiplexing gains; this prevents 
extreme reward sparsity across multi-hop topologies.

\textbf{Endpoint sampling.} $(v_s, v_g)$ are sampled 
uniformly from node pairs separated by topological 
distance $\ge 3$, ensuring every instance requires 
combinatorial routing.

The complete node mapping is given in 
Table~\ref{tab:surfnet_nodes} and the bidirectional edges 
in Table~\ref{tab:surfnet_edges}.

\begin{table}[h]
\centering
\caption{Node index to city mapping for the 25-node SURFnet abstraction.}
\label{tab:surfnet_nodes}
\begin{tabular}{cl | cl | cl}
\toprule
\textbf{Index} & \textbf{City} & \textbf{Index} & \textbf{City} & \textbf{Index} & \textbf{City} \\
\midrule
0 & Amsterdam & 9 & Nijmegen & 18 & Apeldoorn \\
1 & Leiden & 10 & Arnhem & 19 & Amersfoort \\
2 & The Hague & 11 & Enschede & 20 & Dordrecht \\
3 & Delft & 12 & Groningen & 21 & Middelburg \\
4 & Rotterdam & 13 & Maastricht & 22 & Den Bosch \\
5 & Utrecht & 14 & Haarlem & 23 & Venlo \\
6 & Breda & 15 & Alkmaar & 24 & Almere \\
7 & Tilburg & 16 & Leeuwarden & & \\
8 & Eindhoven & 17 & Zwolle & & \\
\bottomrule
\end{tabular}
\end{table}

\begin{table}[h]
\centering
\caption{Bidirectional edges and nominal fiber lengths ($L_{\text{nom}}$) bounded by quantum constraints.}
\label{tab:surfnet_edges}
\begin{tabular}{llr | llr}
\toprule
\textbf{Node $u$} & \textbf{Node $v$} & \textbf{L (km)} & \textbf{Node $u$} & \textbf{Node $v$} & \textbf{L (km)} \\
\midrule
0 (Amsterdam) & 1 (Leiden) & 40 & 14 (Haarlem) & 15 (Alkmaar) & 35 \\
0 (Amsterdam) & 5 (Utrecht) & 40 & 0 (Amsterdam) & 15 (Alkmaar) & 40 \\
1 (Leiden) & 2 (The Hague) & 20 & 0 (Amsterdam) & 24 (Almere) & 30 \\
2 (The Hague) & 3 (Delft) & 15 & 5 (Utrecht) & 19 (Amersfoort) & 25 \\
3 (Delft) & 4 (Rotterdam) & 15 & 19 (Amersfoort) & 24 (Almere) & 45 \\
4 (Rotterdam) & 5 (Utrecht) & 60 & 24 (Almere) & 17 (Zwolle) & 70 \\
4 (Rotterdam) & 6 (Breda) & 55 & 19 (Amersfoort) & 18 (Apeldoorn) & 45 \\
5 (Utrecht) & 8 (Eindhoven) & 85 & 19 (Amersfoort) & 17 (Zwolle) & 65 \\
5 (Utrecht) & 9 (Nijmegen) & 80 & 18 (Apeldoorn) & 17 (Zwolle) & 40 \\
5 (Utrecht) & 10 (Arnhem) & 60 & 18 (Apeldoorn) & 10 (Arnhem) & 30 \\
6 (Breda) & 7 (Tilburg) & 25 & 18 (Apeldoorn) & 11 (Enschede) & 70 \\
7 (Tilburg) & 8 (Eindhoven) & 35 & 17 (Zwolle) & 12 (Groningen) & 100 \\
8 (Eindhoven) & 9 (Nijmegen) & 85 & 17 (Zwolle) & 16 (Leeuwarden) & 90 \\
8 (Eindhoven) & 13 (Maastricht) & 90 & 16 (Leeuwarden) & 12 (Groningen) & 60 \\
9 (Nijmegen) & 10 (Arnhem) & 15 & 4 (Rotterdam) & 20 (Dordrecht) & 25 \\
10 (Arnhem) & 11 (Enschede) & 80 & 20 (Dordrecht) & 6 (Breda) & 35 \\
11 (Enschede) & 12 (Groningen) & 100 & 20 (Dordrecht) & 21 (Middelburg) & 95 \\
0 (Amsterdam) & 14 (Haarlem) & 20 & 6 (Breda) & 21 (Middelburg) & 90 \\
1 (Leiden) & 14 (Haarlem) & 30 & 5 (Utrecht) & 22 (Den Bosch) & 55 \\
22 (Den Bosch)& 8 (Eindhoven) & 35 & 22 (Den Bosch) & 7 (Tilburg) & 25 \\
9 (Nijmegen) & 23 (Venlo) & 65 & 8 (Eindhoven) & 23 (Venlo) & 60 \\
23 (Venlo) & 13 (Maastricht) & 80 & & & \\
\bottomrule
\end{tabular}
\end{table}

\section{Greedy Policy Extraction by Algorithm}
\label{app:greedy_extraction}

The greedy policy at episode $t$ is extracted as follows:
\begin{itemize}
    \item \optinll{} and CombCascade: Dijkstra on the 
    empirical means $\bar p_e(t)$ (no UCB bonus).
    \item \bayesnll{}: Dijkstra on the posterior means 
    $\mathbb{E}[\tilde p_e(t)] = \alpha_e(t)/(\alpha_e(t) + \beta_e(t))$.
    \item $Q$-Learning: the greedy path under current 
    $Q$-values.
    \item RTDP and LRTDP: the greedy path extracted from 
    the current value function.
\end{itemize}
This isolates each algorithm's underlying policy from 
exploration noise.

\section{Baseline Implementation Details}
\label{app:baselines}

This appendix gives full implementation details for the 
six baselines used in Section~\ref{sec:experiments}. The 
main text gives short identifiers; here we record 
hyperparameters, initialization schemes, and protocol 
choices.

\subsection{CombCascade}
\label{app:baselines_combcascade}

CombCascade~\cite{kveton2015combinatorial} is the 
canonical CCB algorithm. It maintains UCB estimates with 
confidence radius $\sqrt{1.5 \ln(t-1) / T_{t-1}(e)}$ and a 
one-sided clamp at $1$. The original algorithm assumes an 
exact maximization oracle for 
$\arg\max_{\pi \in \Psi(v_s, v_g)} \prod_{e \in \pi} 
\hat p_e$; we implement this oracle via Dijkstra on the 
log-transformed weights $-\log \hat p_e$, identical to 
\optinll{}. The implementation differences from 
\optinll{} (initialization, confidence radius constant, 
clamping) are detailed in 
Appendix~\ref{app:combcascade_diffs}.

\subsection{RTDP}
\label{app:baselines_rtdp}

RTDP~\cite{Barto1993rtdp} is a stochastic-shortest-path 
solver that runs simulated trials from $v_s$ and updates 
state values via Bellman backups along the trajectory, 
biasing computation toward states reachable under the 
current greedy policy. We use a goal reward of $+1$ on 
absorption at $v_g$ and $0$ elsewhere; values are 
initialized via the hop-count heuristic 
$V_0(s) = \gamma^{d(s, v_g)}$ for consistency with 
$Q$-Learning, where $d(s, v_g)$ is the unweighted 
shortest hop count from $s$ to $v_g$.

\subsection{LRTDP}
\label{app:baselines_lrtdp}

LRTDP~\citep{Bonet2003lrtdp} extends RTDP with a 
convergence-labeling mechanism: states whose Bellman 
residual falls below $\varepsilon = 10^{-3}$ and whose 
greedy successors are all similarly converged are labeled 
\emph{solved} and excluded from subsequent trials. LRTDP 
is generally considered a strict improvement over RTDP. 
Reward and initialization follow RTDP.

\subsection{Q-Learning}
\label{app:baselines_qlearning}

Q-Learning is a standard tabular off-policy RL baseline 
with learning rate $\alpha_{QL} = 0.1$, discount factor 
$\gamma = 0.99$, and $\varepsilon$-greedy exploration. 
The exploration rate decays exponentially per episode as 
$\varepsilon_t = \max(0.05,\, 0.999^t)$, reaching the 
floor at $t \approx 3000$ episodes. Rewards are $+1$ on 
reaching $v_g$ and $-1$ on edge failure. Action-values are 
initialized via the hop-count heuristic 
$Q_0(s, a) = \gamma^{d(s, v_g)}$, where $d(s, v_g)$ is the 
unweighted shortest hop count from $s$ to $v_g$ (an 
optimistic initialization that encourages systematic 
exploration). The use of a discount factor and negative 
failure rewards introduces a slight objective bias 
relative to the pure \srp{} formulation, but reflects 
standard ``off-the-shelf'' RL deployment.

\subsection{CUCB}
\label{app:baselines_cucb}

CUCB (Combinatorial UCB without log-transform) isolates 
the contribution of the log-transform: it uses the same 
UCB estimates and Dijkstra solver as \optinll{}, but runs 
Dijkstra on the weights $w(e) = 1 - \hat p_e(t)$ rather 
than $-\log \hat p_e(t)$. Since minimizing 
$\sum_e (1 - \hat p_e)$ is not equivalent to maximizing 
$\prod_e \hat p_e$ except in the high-reliability limit 
$\hat p_e \to 1$ (with discrepancy growing in path 
length), CUCB optimizes the wrong objective. We include 
it to confirm that the log-transform is responsible for 
the optimality of Lemma~\ref{lem:structural} rather than 
a stylistic detail.

\subsection{Random}
\label{app:baselines_random}

Random selects a path via a random walk with cycle 
avoidance and dead-end restart: at each step, a neighbor 
is chosen uniformly at random among unvisited outgoing 
neighbors; if all neighbors are already visited, the walk 
restarts from $v_s$, retrying up to 1000 times. If no 
valid path is found, the episode is skipped and counts as 
a failure. This provides a worst-case performance floor.

\section{Additional Experimental Results}
\label{app:additional_experiments}
This appendix collects supplementary experimental results referenced in Section~\ref{sec:experiments}.

\subsection{Experiment 1: Bound Comparison}
\label{app:bound_comparison}
Figure~\ref{fig:bounds_comparison} compares the edge-level 
CombCascade bound (Eq.~\ref{eq:combcascade_bound}) and the 
path-level \optinll{} bound 
(Theorem~\ref{thm:path_dependent_regret}) against the 
empirical \optinll{} regret. The interpretation is given 
in Section~\ref{sec:exp1}.

\begin{figure}[h!]
\centering
\caption{Empirical comparison of regret bounds for 
\optinll{} across four domains. Each plot shows three 
curves on a $\log_{10}$ y-axis: the edge-level CombCascade 
bound (Eq.~\ref{eq:combcascade_bound}), the path-level 
\optinll{} bound (Theorem~\ref{thm:path_dependent_regret}),
and the empirical \optinll{} cumulative regret. The x-axis 
is the episode index $t$.}
\label{fig:bounds_comparison}
\begin{subfigure}[b]{0.24\textwidth}
    \centering
    \includegraphics[width=\textwidth, trim=0 0 0 5mm, clip]{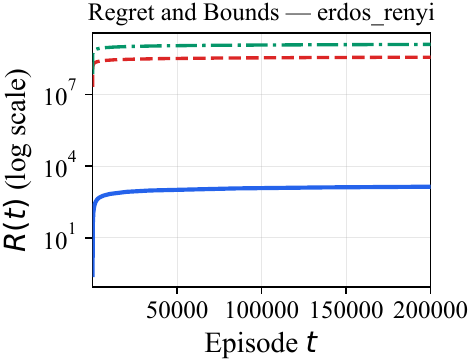}
    \caption{Erd\H{o}s-R\'{e}nyi}
    \label{fig:bounds_erdos_renyi}
\end{subfigure}
\hfill
\begin{subfigure}[b]{0.24\textwidth}
    \centering
    \includegraphics[width=\textwidth, trim=6mm 0 0 5mm, clip]{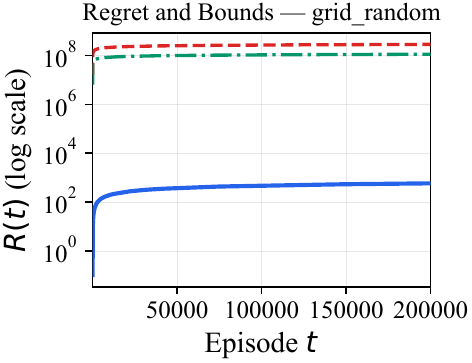}
    \caption{Grid World}
    \label{fig:bounds_grid_random}
\end{subfigure}
\hfill
\begin{subfigure}[b]{0.24\textwidth}
    \centering
    \includegraphics[width=\textwidth, trim=6mm 0 0 5mm, clip]{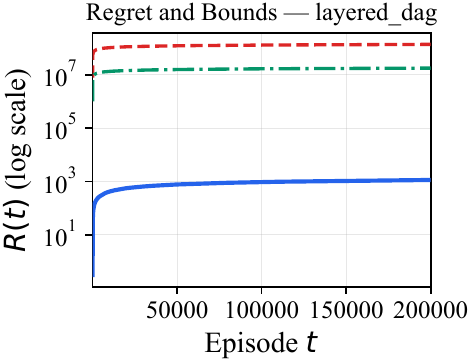}
    \caption{Layered DAG}
    \label{fig:bounds_layered_dag}
\end{subfigure}
\hfill
\begin{subfigure}[b]{0.24\textwidth}
    \centering
    \includegraphics[width=\textwidth, trim=6mm 0 0 5mm, clip]{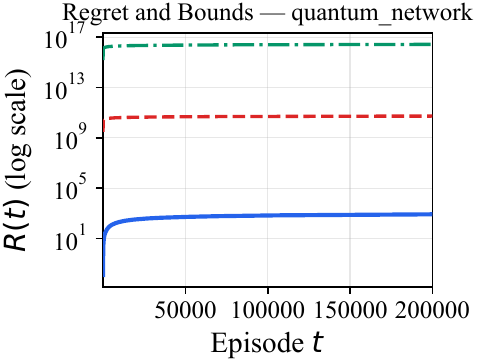}
    \caption{Quantum Network}
    \label{fig:bounds_quantum_network}
\end{subfigure}
\makebox[\textwidth][c]{%
    \includegraphics[width=1.0\textwidth]{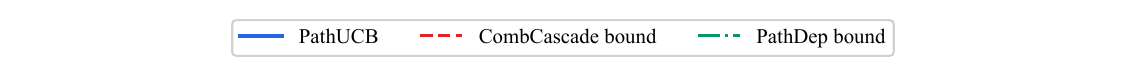}
}
\end{figure}

\subsection{Experiment 1: Full Numerical Results}
\label{app:exp1_table}

Table~\ref{tab:accumulated_regret} gives final cumulative 
regret, convergence rate, and log-linear fit parameters 
for every (algorithm, domain) pair in Experiment~1. The domain sizes where set as follows: 
\begin{itemize}
    \item \textbf{Erd\H{o}s-R\'{e}nyi:} $|V|=12$.
    \item \textbf{Layered DAG:} $k=4$, $w=3$.
    \item \textbf{Grid World:} Grid size $=3 \times 3$
    \item \textbf{Quantum Network:} As described in Appendix~\ref{app:surfnet}.
\end{itemize}

\begin{table}[t]
\centering
\caption{Cumulative regret at $T = 200{,}000$, averaged over 
10 topologies $\times$ 20 seeds (200 runs) per (algorithm, 
domain) pair. Graph size $|V|$ is shown in parentheses next 
to each domain. \optinll{} uses the theoretical exploration 
parameter $\rho = 2.0$; \bayesnll{}, CombCascade, and the 
baselines use their default settings. \emph{Conv.\ Rate} is 
the fraction of runs in which the greedy policy matched some 
$\pi^* \in \Psi^*$ for 10 consecutive episodes within $T$. 
\emph{Log-fit $a$} and $R^2$ are the slope and 
goodness-of-fit of $\mathbb{E}[R(T)] \approx a \ln T + b$ at 
40 log-spaced episodes.}
\label{tab:accumulated_regret}
\begin{tabular}{l l r r r r r}
\toprule
Graph Type & Agent & Final $R(T)$ & ($\pm$std) & Conv.\ Rate & Log-fit $a$ & $R^2$ \\
\midrule
\multirow{7}{*}{Erd\H{o}s-R\'{e}nyi} & PathTS & $99$ & $\pm 86$ & 100\% & 16 & 0.953 \\
 & CombCascade & $1,067$ & $\pm 303$ & 100\% & 169 & 0.995 \\
 & PathUCB & $1,382$ & $\pm 399$ & 100\% & 226 & 0.994 \\
 & Q-Learning & $8,916$ & $\pm 3,078$ & 100\% & 2,205 & 0.754 \\
 & RTDP & $31,025$ & $\pm 38,341$ & 38\% & 7,785 & 0.752 \\
 & LRTDP & $40,946$ & $\pm 43,472$ & 28\% & 10,275 & 0.752 \\
 & Random & $83,249$ & $\pm 28,353$ & 0\% & 20,893 & 0.752 \\
\midrule
\multirow{7}{*}{Layered DAG} & PathTS & $101$ & $\pm 102$ & 100\% & 15 & 0.992 \\
 & CombCascade & $891$ & $\pm 374$ & 100\% & 168 & 0.975 \\
 & PathUCB & $1,138$ & $\pm 457$ & 100\% & 220 & 0.971 \\
 & Q-Learning & $6,303$ & $\pm 2,196$ & 100\% & 1,551 & 0.754 \\
 & RTDP & $19,636$ & $\pm 26,878$ & 40\% & 4,926 & 0.752 \\
 & LRTDP & $31,659$ & $\pm 29,972$ & 22\% & 7,945 & 0.752 \\
 & Random & $60,056$ & $\pm 33,546$ & 0\% & 15,072 & 0.752 \\
\midrule
\multirow{7}{*}{Quantum Network} & PathTS & $37$ & $\pm 23$ & 100\% & 6 & 0.995 \\
 & CombCascade & $660$ & $\pm 438$ & 100\% & 143 & 0.968 \\
 & PathUCB & $835$ & $\pm 574$ & 100\% & 188 & 0.960 \\
 & Q-Learning & $2,104$ & $\pm 1,230$ & 100\% & 512 & 0.799 \\
 & Random & $4,819$ & $\pm 5,128$ & 0\% & 1,209 & 0.752 \\
 & RTDP & $5,292$ & $\pm 6,435$ & 2\% & 1,328 & 0.752 \\
 & LRTDP & $5,359$ & $\pm 6,418$ & 6\% & 1,345 & 0.752 \\
\midrule
\multirow{7}{*}{Grid World} & PathTS & $42$ & $\pm 33$ & 100\% & 6 & 0.997 \\
 & CombCascade & $472$ & $\pm 190$ & 100\% & 93 & 0.965 \\
 & PathUCB & $588$ & $\pm 232$ & 100\% & 119 & 0.962 \\
 & Q-Learning & $4,929$ & $\pm 1,757$ & 98\% & 1,213 & 0.756 \\
 & RTDP & $19,100$ & $\pm 16,932$ & 20\% & 4,793 & 0.752 \\
 & LRTDP & $22,381$ & $\pm 17,748$ & 16\% & 5,616 & 0.752 \\
 & Random & $22,729$ & $\pm 13,399$ & 0\% & 5,704 & 0.752 \\
\bottomrule
\end{tabular}
\end{table}

\subsection{Experiment 3: Scaling Analysis}
\label{app:exp3}

We study how performance scales with graph size 
($|V|$, $|E|$) and edge-reliability range. Since 
\optinll{} with $\rho = 1$ outperformed CombCascade on 
all five domains in Experiment~2, we focus on \optinll{} 
and \bayesnll{}.

\subsubsection{Scaling with Graph Size}
\label{app:scaling_graph}

We sweep two structurally different axes 
(Figure~\ref{fig:scaling_full}a,b): Erd\H{o}s--R\'enyi 
with $|V| \in \{6, 8, 10, 12, 15, 18, 20, 25, 30\}$ at 
edge probability $0.4$ (path count super-polynomial in 
$|V|$), and Layered DAG with a fixed 6-layer spine and 
width $w \in \{2, \dots, 8\}$ (path count $w^6$, 
polynomial in $w$). We use $T = 100{,}000$ episodes.

On both domains, \optinll{} and \bayesnll{} exhibit 
approximately linear regret growth in $|E|$. In dense 
graphs $|E| = O(|V|^2)$, so this translates to polynomial 
scaling in $|V|$.

\subsubsection{Scaling with Edge Reliability Range}
\label{app:scaling_pmin}

On Erd\H{o}s--R\'enyi graphs with $|V| = 15$, we sweep
$p_{\min} \in \{0.05, 0.15, 0.25, 0.35, 0.45, 0.55, 0.59\}$
with $p_{\max} = p_{\min} + 0.4$. The support span $p_{\max}-p_{\min}=0.4$ is held fixed throughout, so the sweep isolates the location of the reliability interval
on $[0,1]$ as the source of variation.
Topologies and seeds are matched across $p_{\min}$ values.

With the span $p_{\max}-p_{\min}$ fixed, increasing
$p_{\min}$ shrinks the relative spread of edge reliabilities
($p_{\max}/p_{\min} \to 1$). Path reliabilities
$P(\pi)=\prod_{e\in\pi} p_e$ therefore cluster, and the per-path gaps
$\Delta(\pi)$ contract. Theorem~\ref{thm:path_dependent_regret} predicts regret scaling
with $1/\Delta(\pi)$, and the number of episodes needed to separate a
suboptimal path scales with $1/\Delta(\pi)^2$ (Lemma~\ref{lem:per_path_selections}), so
the observed super-linear growth of \optinll{} is consistent with the
gap-dependent terms of the bound. This effect is partially offset by improved
observability: higher $p_{\min}$ raises the prefix reliabilities $Q_i(\pi)$,
so deeper edges are observed more often per episode. The net trend in
Figure~\ref{fig:scaling_full}c indicates that gap contraction dominates.

\begin{figure}[t]
\centering
\caption{Final cumulative regret $R(T)$ at 
$T = 100{,}000$ versus graph size on Erd\H{o}s--R\'enyi 
(a) and Layered DAG (b), and versus $p_{\min}$ on 
Erd\H{o}s--R\'enyi with $|V| = 15$ (c).}
\label{fig:scaling_full}
\begin{subfigure}[b]{0.32\textwidth}
    \centering
    \includegraphics[width=\textwidth]{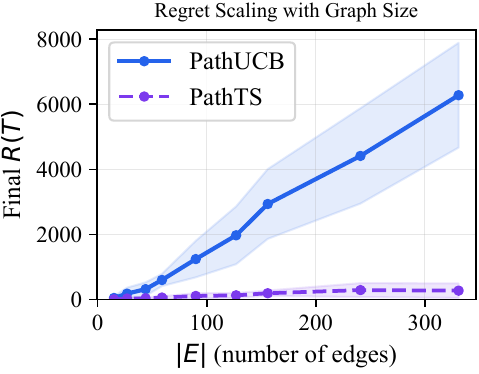}
    \caption{ER: $|E|$}
    \label{fig:scaling_er_edges}
\end{subfigure}
\hfill
\begin{subfigure}[b]{0.32\textwidth}
    \centering
    \includegraphics[width=\textwidth]{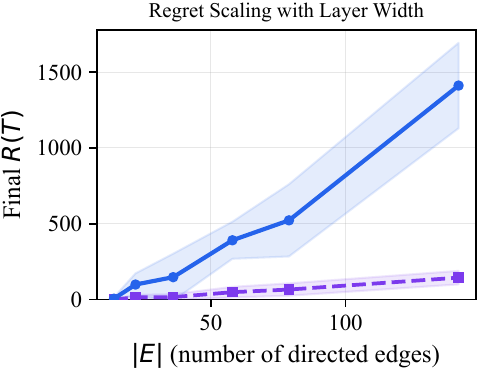}
    \caption{DAG: $|E|$}
    \label{fig:scaling_dag_edges}
\end{subfigure}
\hfill
\begin{subfigure}[b]{0.32\textwidth}
    \centering
    \includegraphics[width=\textwidth]{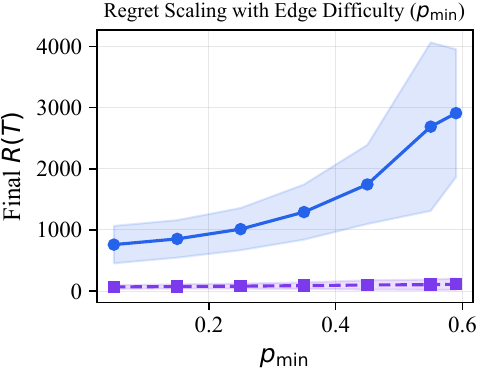}
    \caption{ER(15): $p_{\min}$}
    \label{fig:scaling_pmin}
\end{subfigure}
\end{figure}

\clearpage
\newpage

\end{document}